\documentclass{article}

\usepackage{microtype}
\usepackage{graphicx}
\usepackage{subfigure}
\usepackage{booktabs} 

\usepackage[accepted]{icml2020}
\usepackage{alp}
\usepackage{tikz}

\usepackage{amsthm}

\usepackage{hyperref}

\usepackage{xspace}
\usepackage{caption}

\newtheorem{theorem}{Theorem}[section]

\newcommand{\overbar}[1]{\mkern 1.5mu\overline{\mkern-1.5mu#1\mkern-1.5mu}\mkern 1.5mu}

\newacronym[plural=LPNNs,firstplural=Ladder Polynomial Neural Networks]{lopon}{LPNN}{Ladder Polynomial Neural Network}

\newcommand{\lopon}{\acrshort{lopon}\xspace}
\newcommand{\lopons}{\acrshortpl{lopon}\xspace}
\newcommand{\loponfunc}{\mathrm{lpnn}}

\icmltitlerunning{Ladder Polynomial Neural Networks}

\begin{document}

\twocolumn[
\icmltitle{Ladder Polynomial Neural Networks}

\icmlsetsymbol{equal}{*}

\begin{icmlauthorlist}
\icmlauthor{Li-Ping Liu}{tuftscs}
\icmlauthor{Ruiyuan Gu}{tuftscs}
\icmlauthor{Xiaozhe Hu}{tuftsmath}
\end{icmlauthorlist}

\icmlaffiliation{tuftscs}{Department of Computer Science, Tufts University}
\icmlaffiliation{tuftsmath}{Department of Mathematics, Tufts University}

\icmlcorrespondingauthor{Li-Ping Liu}{liping.liu@tufts.edu}

\icmlkeywords{polynomial, learning models, classification}

\vskip 0.3in
]

\printAffiliationsAndNotice{The work has been first submitted to ICLR 2019 (\href{https://openreview.net/forum?id=HJxN0CNFPB}{submission link}). Unfortunately the contribution was not sufficiently appreciated by reviewers. \\ }

\begin{abstract}
Polynomial functions have plenty of useful analytical properties, but they are rarely used as learning models because their function class is considered to be restricted. This work shows that when trained properly polynomial functions can be strong learning models. Particularly this work constructs polynomial feedforward neural networks using the \textit{product activation}, a new activation function constructed from multiplications.  The new neural network is a polynomial function and provides accurate control of its polynomial order.  It can be trained by standard training techniques such as batch normalization and dropout. This new feedforward network covers several previous polynomial models as special cases. Compared with common feedforward neural networks, the polynomial feedforward network has closed-form calculations of a few interesting quantities, which are very useful in Bayesian learning. In a series of regression and classification tasks in the empirical study, the proposed model outperforms previous polynomial models.

\end{abstract}

\section{Introduction}
Well studied by mathematicians, polynomial functions have many favorable theoretical properties. Polynomial models also bridge the analysis of general neural networks to the properties of polynomial functions  \citet{livni2014computational}. For example, polynomial models can approximate other feedforward neural networks, and they are polynomial-time learnable.

One method of constructing polynomial models is to use the quadratic function as activations in a feedforward network (\acrshort{ffq}). However, it has a clear drawback: an \acrshort{ffq} cannot have an arbitrary polynomial order, as its order grow exponentially with its number of layers. As a result, a deep \acrshort{ffq} is hard train, and its performance is not stable.

Polynomial learning models can also be devised by representing polynomial coefficients with some type of decomposition. We simply call these models as \emph{decomposition} models. One approach is to define a polynomial kernel over input features and network parameters \citep{blondel2016polynomial, blondel2016higher, blondel2017multi}. Another approach is to use a tensor-train decomposition as the coefficients of a polynomaial model \citet{chen2017}. These models do not have a layer structure, so they often need specialized training methods. Furthermore, it is very hard to compare these models with neural networks, so some of their good properties are not well understood. 

In this work, we propose a new method of constructing polynomial learning models in the form of feedforward networks. One key component of a feedforward network is the activation function. We devise the \emph{product activation}, which creates nonlinearity by multiplying a hidden layer to a linear transform of the. Then we construct a \acrfull{lopon} with product activations. The \lopon has a layer structure by construction,  but it is also a decomposition model at the same time. Therefore, it enjoys benefits from both sides. As a feedforward neural network, it can be trained by standard deep learning techniques such as batch normalization \citep{ioffe2015batch} and dropout \citep{srivastava2014dropout}. Its polynomial order, which is the number of hidden layers plus 1, can be exactly controlled. As a decomposition model, \lopon also covers two previous models as special cases. When network weights are stochastic, the moments of a \lopon's outputs can be computed in closed-form. This property is very useful in  Bayesian learning. 

The empirical study shows that the \lopon outperforms previous polynomial models in a list of classification and regression tasks. The investigation also indicates the necessity of dropout and batch normalization in training. In the setting of Bayesian learning, we show that Gaussian distributions can well approximate a \lopon's network output when the network is given a Gaussian prior.

\section{Related Work}
\parhead{Feedforward network with quadratic activations.} \citet{livni2014computational} analyze \acrshortpl{ffq} with quadratic activations and show several positive properties of polynomial neural networks. For example, they are as expressive as networks with threshold activations, and they are learnable in polynomial time. \citet{kileel2019expressive} analyze of the algebraic structure of polynomial functions behind \acrshort{ffq}. \citet{du2018power} show that training a one-hidden-layer \acrshort{ffq} is efficient when the model is overly parameterized. All these analyses depend on the special function form. There are also other special optimization methods \citep{lin2017second, soltani2018towards, soltani2019fast} for training \acrshort{ffq}s with one hidden layer. 

\parhead{Decomposition models.}
\citet{blondel2016polynomial} construct polynomial models with polynomial kernels. They also show that factorization machines \citep{rendle2010factorization} can be constructed in the same way with ANOVA kernels. \citet{blondel2016higher} propose high order factorization machines with high order ANOVA kernels. \citet{blondel2017multi} extend factorization machines and polynomial networks to output multiple values. \citet{chen2017} use tensor-train decomposition \citep{oseledets2011} to express the coefficients of a polynomial model. By design, the model is for small problems. The coefficients of the \lopon in this work also has a tensor-train \citep{oseledets2011}. 



\section{The Polynomial Neural Network}

We first define the general form of a feedforward neural network. Suppose the input to the neural network is a feature vector $\bx \in \bbR^{d_0}$ and denote $\bh^{0} = \bx$. Suppose the network has $L$ hidden layers, with each layer $\ell \in \{1, \ldots, L\}$ takes the input $\bh^{\ell - 1}$ and has the output $\bh^{\ell}$. Each layer is defined by  
\begin{align}
\bh^{\ell} = \sigma \left(\bW^{\ell}  \bh^{\ell - 1} \right). \label{eq:layer}
\end{align}
Here $\bW^{\ell}$ is the weight matrix for layer $\ell$, and $\sigma(\cdot)$ is the activation function. For notational simplicity, we omit intercept vectors for now and will include them later. 

We use a new activation, the \textit{product activation} $\sigma_{p}(\cdot)$ in the neural network.
\begin{align}
\sigma_{p}(\bu; \bV, \bx) = \bu \odot(\bV \bx). \label{eq:prod_act}
\end{align}
Here $\odot$ is the element-wise product. The learnable parameter $\bV$ is a matrix with size $(d \times d_0)$ when $\bu$ has $d$ entries.

Since $\bu=\bW^{\ell}\bh^{\ell - 1}$ is a function of $\bx$, the activation is nonlinear in $\bu$. Particularly, if $\bu$ is a polynomial function of $\bx$, then $\sigma_{p}(\bu; \bV, \bx)$ is also a polynomial function of $\bx$ with the polynomial order increased by 1. Note that the product activation is not a function of $\bu$ because different $\bx$ values may give the same $\bu$ value but different responses from $\sigma_{p}(\bu; \bV, \bx)$. 

The product activation is inspired by self-attention \citep{vaswani2017}, in which the multiplication (of hidden vectors and attention weights) is an important way of processing information. The product activation keeps the multiplication operation and removes all non-linear operations. 

Then we use product activations in a feedforward structure to construct a \lopon.  We use a different matrix $\bV^{\ell}$ for the product activation in each layer $\ell$. Suppose $\bh_L$ is the output of the neural network, the function of the \lopon ~ is formally defined as $\loponfunc(\bx; \theta) = \bh^{L}$, 
\begin{align}
& \bh^{0} =  \bx, \\
& \bh^{\ell} = \bW^\ell \bh^{\ell-1} \odot (\bV^{\ell} \bx), ~~ \ell = 1, 2, \ldots, L.\label{eq:h_ell}
\end{align}
Here $\odot$ denotes element-wise multiplication. Let $\theta = \left(\bW^1, \ldots, \bW^L, \bV^1, \ldots, \bV^L\right)$ denote all network parameters. 
The first hidden layer $\bh^{1}$ is a second-order polynomial 
of the input, and each activation increase the order by 1, so the hidden layer $\bh^{\ell}$ is an order $(\ell+1)$ polynomial. 

We further re-write the function with simple additions and multiplications. The $i$-th entry of $\bh^{\ell}$ is 
\begin{align}
h^{\ell}_i =  \left( \bW_i^{\ell} \bh^{\ell} \right)  \left(\bV^{\ell}_i \bx \right). \label{eq:layer_alt}
\end{align}
Here $\bW^{\ell}_i$ and $\bV^{\ell}_i$ are the $i$-th rows of $\bW$ and $\bV$ respectively, so both $\left( \bW_i^{\ell} \bh^{\ell} \right)$ and $\left(\bV^{\ell}_i \bx \right)$ are scalars. 

After expanding all product activations and taking out all summations, the $i_L$-th entry of the network function $\loponfunc(\bx; \theta)$ is
\begin{multline}
h^{L}_{i_L} = \sum_{i_{(L-1)} = 1}^{d_{(L-1)}}\cdots 
\sum_{i_{1}=1}^{d_1} \Bigg[\left( \prod_{\ell = 2}^L \bW_{i_\ell, i_{\ell-1}}^{\ell}
\right) \\ \left({\bW}_{i_1}^{1} \bx \right) \left(\prod_{\ell = 1}^{L} \bV^{\ell}_{i_\ell} \bx \right) \Bigg].
\label{eq:network_func}
\end{multline}
This equation further show that the polynomial order of the network is $L+1$. 

To include terms with different orders, we need to include intercept vectors. Suppose each layer has an intercept vector $\bb^{\ell}$, then the layer is defined by
\begin{align}
\bh^{\ell} = \sigma_p(\bW^{\ell} \bh^{\ell - 1} + \bb^{\ell}; \bV^{\ell}, \bx).
\label{eq:layer_intercept}
\end{align}
We can re-write the input, hidden vectors, and weight matrices in the following form, 
\begin{align}
&\hat{\bW}^{\ell} = 
\left[\begin{array}{cc}
\bW^{\ell} & \bb^{\ell} \\ 
\bzero^\top &  1
\end{array}\right], ~~ 
\hat{\bh}^{\ell} = \
\left[\begin{array}{c}
\bh \\
1
\end{array}\right], ~~ \label{eq:intercept1}\\
&\hat{\bV}^{\ell} = 
\left[\begin{array}{cc}
\bV^{\ell} & \bzero\\ 
\bzero^\top &  1
\end{array}\right], ~~
\hat{\bx} = \
\left[\begin{array}{c}
\bx \\
1
\end{array}\right],
\label{eq:intercept_absorb}
\end{align}
then we still have the previous form, $\hat{\bh}^{\ell} = \sigma_p(\hat{\bW}^{\ell} \hat{\bh}^{\ell-1}; \hat{\bV}^{\ell}, \hat{\bx})$, and then all previous derivations still apply. For notational simplicity, the following analysis continues to use notations without the intercept term. 

The \lopon shares the same principle as the ResNet \citep{he2016deep}: a deeper networks should not have larger training error than a shallower one. Similar to skip connections, which allows a deep ResNet to implement a shallower one, a special parameter setting of $\theta$ also reduces a \lopon to a shallower one. Setting $\bW^{\ell}=\bzero$ and $\bb^\ell = \bone$ in (\ref{eq:intercept1}) will reduce the \lopon to $L - \ell$  layers.

\section{Analysis}
In this section, we first connect \lopon with a few previous decomposition models, highlighting a few properties of \lopon as a polynomial function. Then we emphasize that \lopon can be trained with batch normalization and dropout. Finally we analyze \lopon from the perspective of neural networks and show its properties in Bayesian learning. 

\subsection{Relation with decomposition models}

By making comparison between a \lopon and previous polynomial models, we understand their respective weakness and strength. 

\parhead{Relation with polynomial kernels.} We first show that the polynomial networks constructed from polynomial kernel functions by \citet{blondel2016polynomial} are special cases of the \lopon model. The following theorem formalize the relationship, and its proof is in the supplement.

\begin{theorem} 
The learning models in the form of $y(\bx) = \sum_{k=1}^K \pi_k (\lambda + \bp_k^\top \bx)^m$ \citep{blondel2016polynomial} can be written as a \lopon function. 
\end{theorem}

Compared with a \lopon, a model constructed from polynomial kernels has a limited capacity. When it is written in the form of a \lopon, a kernel is used across all hidden layers, so it form is rather restricted. The model with multiple kernels has multiple hidden units, but there is no information exchange between hidden units from different kernels.  

\parhead{Relation with factorization machines.}
By the following theorem, second-order factorization machines \citep{rendle2010factorization} are special cases of \lopons. The proof is in the supplement. 
\begin{theorem}
The second-order factorization machines taking the form
$y(\bx) = w_0 + \bw_1^\top \bx + \sum_{i=1}^{d_0}\sum_{j =i+1}^{d_0} \bv_i^\top \bv_j x_i x_j$ \citep{rendle2010factorization} can be written as a \lopon function.   
\end{theorem}

Factorization machines execlude monomials that contain variables having exponents more than 1, e.g. $x_i^2$ in the example above, so it is hard to write high-order factorization machines into the \lopon form. However, it is easy to construct a \lopon to match all its non-zero coefficents. In this sense, factorization machines have less model capacity than \lopon. 

\parhead{Relation with tensor-train models.}
The coefficients of \lopon has a tensor-train decomposition \citep{oseledets2011}, so it is a special case of the tensor-train model \citep{chen2017}.
\begin{theorem}   
The coefficients of \lopon's network function has a tensor-train decompositoin.  
\end{theorem}
\begin{proof}
Write $\bW^{1}_{i_1} \bx  =\sum_{i_0 = 1}^{d_0} \bW^{1}_{i_1, i_0} \bW^{0}_{i_0} \bx_{i_0}$ with $\bW^0$ being an identity matrix. Then \eqref{eq:network_func} can be written as 
\begin{multline*}
\bh^{L}_{i_L} \hspace{-0.1em} = \hspace{-0.1em} \sum_{i_{L-1} = 1}^{d_{L-1}} \hspace{-0.5em}\cdot\cdot\cdot 
\sum_{i_{1}=1}^{d_1}\sum_{i_{0}=1}^{d_0} \left[\left({\bW}_{i_0}^{0} \bx \right) \prod_{\ell = 1}^L \bW_{i_\ell, i_{\ell-1}}^{\ell} \bV^{\ell}_{i_\ell} \bx \right].
\label{eq:tt_form}
\end{multline*}

Let $\bG^{\ell} (i_\ell, n_\ell,  i_{\ell - 1}) =  \bW^{\ell}_{i_\ell, i_{\ell} - 1} \bV^{\ell}_{i_{\ell}, n_\ell}$ be a three-way tensor for $\ell = 1, \dots, L - 1$. Let $\bG^{L}$ be a three-way tensor with size 1 in the first dimension, and $\bG^{L}(1, n_L, i_{\ell - 1}) = \bW^{L}_{i_L, i_{L-1}} \bV_{i_{L}, n_{L}}$. Let $\bG^{0}$ be a three-way tensor with size 1 in the last dimension, and $\bG^{0}(:, :, 1) = \bI$. Then $\bh^{L}_{i_L} = \prod_{\ell = 0}^{L} (\bG^{\ell} \times_2 \bx)$. Here $\times_2$ is the 2-mode product of a three-way tensor and a vector; the product represents matrix multiplications. By Lemma 1 in \citep{chen2017}, the coefficents of $\bh^{L}_{i_L}$ is the tensor-train decomposition expressed by $\bG^0, \ldots, \bG^{L}$. 
\end{proof}

We can analyze the complexity of a \lopon and a tensor-train model. In the decomposition above, each three-way tensor $\bG^{\ell}$ with $\ell \ge 1$ is constructed from two matrices. A \lopon layer has about $(d_\ell \times d_{\ell - 1} + d_\ell \times d_{0})$ parameters, which is at the same level as a feedforward layer. However, we use general three-way tensors in a tensor-train model, then there will be  $ d_\ell \times d_{0} \times d_{\ell - 1}$ parameters in each ``layer'' and excessively many parameters in the entire model. Then it is unnecessary to use a dense $\bG^{\ell}$-s as in a general tensor-train model. 

\subsection{Special properties}

As a polynomial function, a \lopon has the following two interesting properties.

\parhead{Multilinear in parameters} The network function $\loponfunc(\bx; \theta)$ is multilinear in model parameters $\theta=\{\bW^1, \ldots, \bW^L, \bV^1, \ldots,  \bV^L\}$. We see so by examining the network function in \eqref{eq:network_func}: if we focus on one matrix (a $\bW^{\ell}$ or a $\bV^{\ell}$) and hold all other parameters fixed, then the network output is linear in this matrix. Furthermore, if we optimize the network against a convex loss function $\mathrm{loss}(\bh^{L}; y)$, then the loss is multiconvex in $\theta$.  

This propery means that the change of a single $\bW^\ell$ linearly change the network output. Therefore, the optimization of network output poses straightforward gradient directions to network weights in $\theta$. There are no issues of gradient explosion. Compared with a \lopon, the effect of network weights in a normal feedforward network are transformed through the network's hidden layers, so the gradient can be skewed when it is back-propagated through layers.

\parhead{The network function along a gradient}. The network function along a gradient direction is a univariate polynomial function. We can write the polynomial into its canonical form. 

 We compute the polynomial coefficients recursively. Let's restrict $\bx$ to a line $\bx=t \bg + \bx_0$, with $t$ being a scalar variable, $\bx_0 \in \mathbb{R}^{d_0}$ being a point, and $\bg \in\mathbb{R}^{d_0}$ being a direction. Then $\loponfunc(\bx; \theta)$ is a polynomial function of $t$ with order $L + 1$. 

We view each $\bh^{\ell}$ as a function of $t$ and use an operation $\alpha(\bh^{\ell})$ to extract polynomial coefficients of $\bh^{\ell}$. Note that each entry of $\bh^{\ell}$ is an $(\ell + 1)$-order polynomial and has $\ell + 2$ coefficients, so $\alpha(\bh^{\ell})$ is a matrix of size $d_\ell \times (\ell + 2)$. Also note that $\alpha(\cdot)$ is a linear operation. 

We have $\alpha(\bh^0) = \alpha(\bx) = [\bg, \bx_0]$. Then we calculate $\alpha(\bh^{\ell})$ recursively. 
Substitue the line expression into \eqref{eq:h_ell}, 
\begin{multline*}
\bh^{\ell} = \diag(\bV^{\ell} \bg) \bW^{\ell} \bh^{\ell - 1} t + \diag(\bV^{\ell} \bx_0) \bW^{\ell} \bh^{\ell - 1}. 
\end{multline*}

The function $\alpha(\cdot)$ is a linear operation. If $\boldf$ is a vector-valued polynomial function of $t$, and $\bW$ is a constant matrix with proper sizes, then 
$\alpha(\bW \boldf) = \bW \alpha(\boldf)$. We also have $\alpha(t \boldf ) = [\alpha(\boldf), \bzero]$, with $\bzero$ being zero vector with the same length as $\boldf$. 
The relation is true because $t \boldf$ raises the coefficients of $\boldf$ one order higher. 

By using these properties of $\alpha(\cdot)$, we have the recursive formula for computing $\alpha(\bh^{\ell})$. 
\begin{multline}
\alpha(\bh^{\ell}) = [\diag(\bV^{\ell} \bg) \bW^{\ell} \alpha(\bh^{\ell - 1}), \bzero] + \\ 
[\bzero,  \diag(\bV^{\ell} \bg) \bW^{\ell} \alpha(\bh^{\ell - 1})].
\end{multline}

If we want compute only the coefficents of lower-order monomials, then we only need to store the right-most few columns of $\alpha(\bh^\ell)$. We omit the details here. 

This property is useful in adversarial learning. The generation of adversarial samples often relies on the optimization of a perturbation of an instance $\bx$. If we know a perturbation direction, then finding the optimal perturbation along the direction is equivalent to finding the minimum of a univariate polynomial function, which can be solved efficiently. 

\subsection{Training with batch normalization and dropout}

As a layer network, a \lopon is trained with standard techniques including stochastic optimization, batch normalization (BN), and dropout, which have been proved to be effective in practice. We can apply batch normalization and dropout to a \lopon without any modification. Here we put the BN layer after the activation per some practitioners' advice. 

Let's consider one hidden layer and omit layer indices for notational simplicity. Let $\bh_{k}$ be the hidden layer of an instance $k$ in a batch, then the batch-normalized hidden layer $\overbar{\bh}_k$ is computed by  
\begin{align}
\overbar{\bh}_{k} = \gamma (\bh_k - \bmu) / (\bsigma + \epsilon)  + \beta.
\label{eq:bn}
\end{align}
Here the division $/$ is an element-wise operation, $\epsilon$ is a small positive number, and $(\gamma$, $\beta)$ are learnable  parameters. The two vectors, $\bmu$ and $\bsigma$, are computed from the batch during training and are constants during testing.

Note that the trained \lopon model with constant BN parameters is still a polynomial function because the BN operation in \eqref{eq:bn} is a linear operation. The model in training is not a polynomial function since $\bsigma$ is computed from the batch that includes $\bh_k$. Here we want to integrate BN parameters into network weights so that previous derivations still apply. 

Now we merge a BN layer into its previous \lopon layer. Let $\bW'$ and $\bV'$ denote weight matrices in the previous layer, then the equivalent \lopon layer is given by setting $\bW$ and $\bb$ in \eqref{eq:layer_intercept} as follows. 
\begin{align}
\bW &= \bW' ~ \diag(\gamma / (\bsigma + \epsilon) ), \quad \\
\bb &=  - \bW' ~ \diag(\gamma \bmu / (\bsigma + \epsilon) + \beta 
\end{align}

In this result, we can see that BN changes the norm of weight matrices. Based on the study by \citet{santurkar2018does}, BN can simply shrink the norms of weight matrices to avoid having steep slopes in the function surface.

Dropout can be directly applied to \lopon. In the training phase, using dropout is equivalent to removing some entries in summations of \eqref{eq:network_func} and rescaling the summation. In the testing phase, dropout have no effect, and the trained model is just as the definition above. 

\subsection{Moments of network outputs in Bayesian learning}

The \lopon is convenient for Bayesian learning attribute to the multilinear property. One important problem in Bayesian learning is to compute the distribution of network predictions when the network parameters are from a distribution $p(\theta)$. 
\begin{align*}
p(y | \bx) &= \int_{\theta} p\big(y ~|~ \bh_L = \loponfunc(\bx; \theta)\big) p(\theta) ~  \mathrm{d}\theta\\
&= \int_{\bh^{L}} p\left(y | \bh_L \right) p(\bh^{L}) ~ \mathrm{d}\bh^{L} 
\end{align*} 
$p(\theta)$ is either a prior or a distribution inferred from the data. 
The integeral is often easy to deal with if we have the distribution $p(\bh_L)$. We propose to approximate $p(\bh_L)$ with a Gaussian distribution, whose parameters can be decided by the moments of $\bh_L$.
 
Assume a prior $p(\theta)$ has all $\bW^{\ell}$ and $\bV^{\ell}$ matrices independent. The assumption is reasonable because the distribution of $\theta$ is often assumed to be a Gaussian distribution with independent componenets \citep{blundell2015}. Then we can compute the first two moments of the network outputs efficiently. 

By the multilinear property, we have the first moment  
\begin{align}
\bmu^{L} = \E{\theta}{\loponfunc(\bx; \theta)} =  \loponfunc\left(\bx; \E{\theta}{\theta}\right). 
\end{align}

We compute the second moment of $\bh^L$ recursively. Denote the second-order moments by $\bSigma^\ell = \E{\theta}{\bh^{\ell}(\bh^{\ell})^\top}$ for $\ell = 1, \ldots, L$. Let $\bSigma^0 = \bx \bx^\top$. The recursive computation of the second order moment of $\bh^{\ell}$ is 

\begin{multline}
\Sigma^{\ell}_{ij} = \E{\theta}{h^{\ell}_i h^{\ell}_j} =  \left(\bx^\top \E{\theta}{(\bV^{\ell}_{i})^{\top}\bV^{\ell}_{j}} \bx  \right) \cdot \\
\trace{\E{\theta}{(\bW_i^{\ell})^{\top} \bW_j^{\ell}} \bSigma^{\ell - 1}} 
\end{multline}

We then approximate the distribution of $\bh^{L}$ by $\mathcal{N}(\bmu^{L}, \bSigma^{L} - \bmu^{L}(\bmu^{L})^{\top})$.  
In the experiment we show that the approximation is very accurate when $p(\theta)$ is a Gaussian distribution. 

In many typical applications, the distribution of network predictions $p(y | \bx)$ is computable when we have an approximation of $p(\bh^{L} | \bx)$. In a regression problem that assumes a Gaussian distribution for $p(y | \bh^{L})$, then the marginal distribution $p(y | \bx)$ has a closed-form approximation. In a binary classification problem where $y = 1[\bh^{L} > 0]$, then approximation of $p(y=1 | \bx)$ is $\Phi(\bh^{L} / \sqrt{var(\bh^{L})})$, with $\Phi(\cdot)$ being the pmf of the standard Gaussain.  

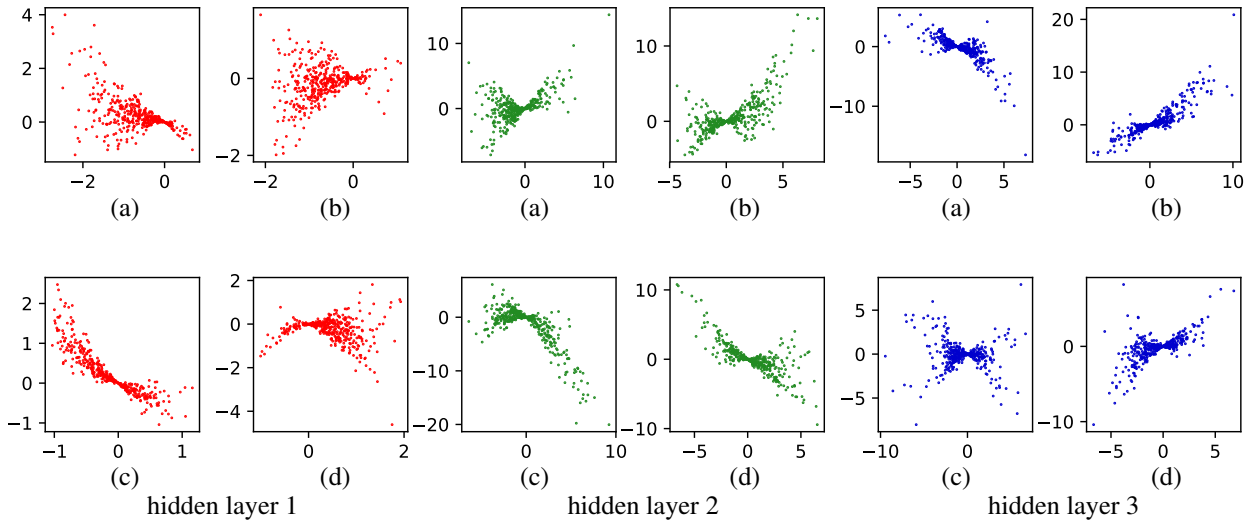
\begin{figure*}[t]
\centering
\begin{tikzpicture} [x=1.3cm,y=1.3cm]
\node at (0,0){\includegraphics[width=\textwidth, trim={2.5cm 1.5cm 2.5cm 2.2cm},clip]{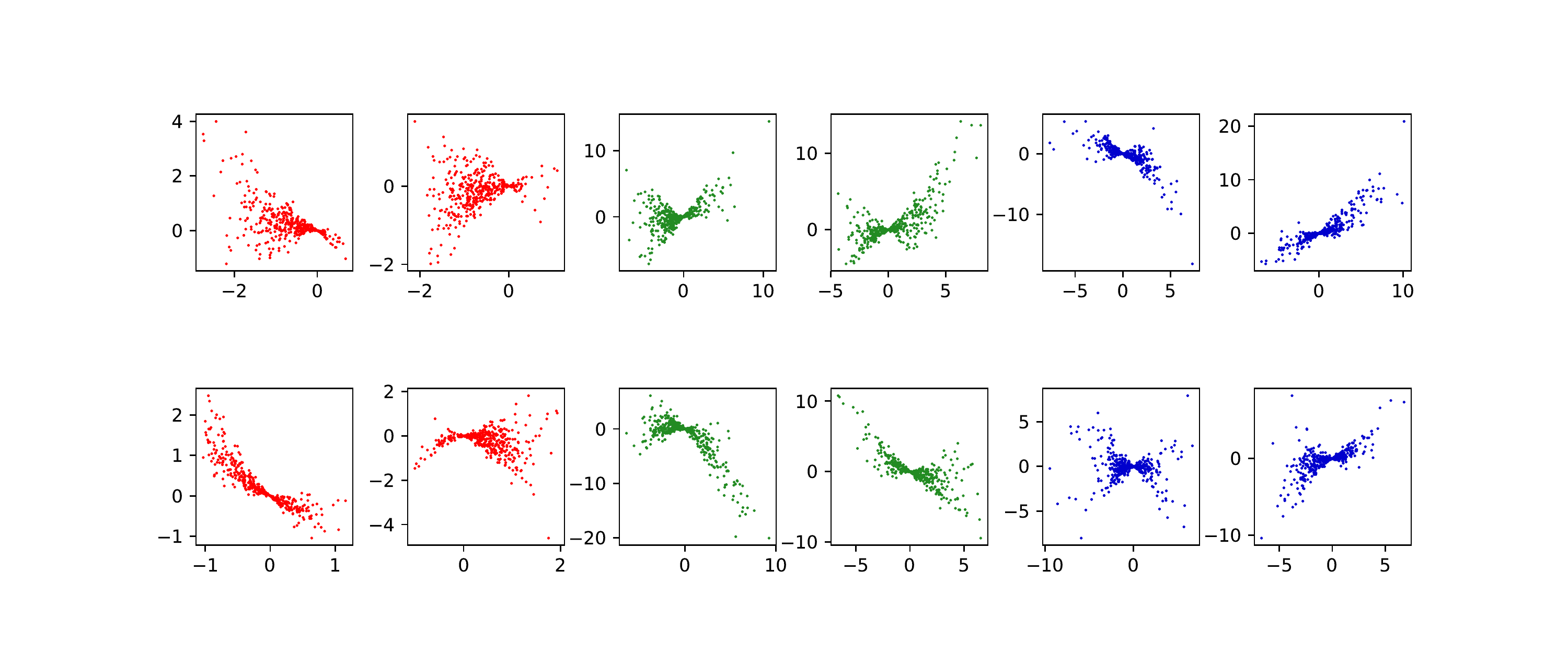}};
\foreach \x/\y/\a in {0/0/a, 1/0/b, 0/1/c, 1/1/d}
    \node at (\x * 2.15 - 5.1, -\y * 2.75 + 0.25) {(\a)};

\foreach \x/\y/\a in {0/0/a, 1/0/b, 0/1/c, 1/1/d}
    \node at (\x * 2.15 - 0.91, -\y * 2.75 + 0.25) {(\a)};

\foreach \x/\y/\a in {0/0/a, 1/0/b, 0/1/c, 1/1/d}
    \node at (\x * 2.15 + 3.4, -\y * 2.75 + 0.25) {(\a)};

\foreach \x in {1,2,3}
    \node at (\x * 4.3 - 8.4, -2.8) {hidden layer \x};
\end{tikzpicture}
\caption{Product activations of LPNN on the mnist dataset. The model has three hidden layers. From each layer, activations of four hidden units are plot here in the same color.}
\label{fig:product_act}
\end{figure*}

\section{Experiment}

\begin{table}
\begin{center}
\caption{Approximate the product activation with a one-layer feedforward neural network . }
\label{tab:approx_prod} 
\scalebox{0.89}{
\begin{tabular}{|r|cccc|}
\hline
\# hidden & 1 & 2 & 3 & 4\\
\hline
product & 1.13 $\pm$ .00 & 0.80 $\pm$ .04 & 0.33 $\pm$.10 & 0.03 $\pm$.01 \\
\hline
ReLU & 0.12 $\pm$ .00 & 0.09 $\pm$ .00 & 0.05 $\pm$.00 & 0.04 $\pm$.01 \\
\hline
\end{tabular}}
\end{center}
\end{table}

\begin{table*}[t]
\caption{RMSE of different models on regression tasks}
\label{tab:regression}
\begin{center}
\scalebox{1.0}{
\begin{tabular}{c|ccccc}
methods & wine red              & power plant          & kin8nm                & boston housing       &  concrete       \\
\hline
\hline
\acrshort{ffr}      & 0.60 $\pm$ 0.04  & 4.02 $\pm$0.18   & 0.100 $\pm$ 0.002~~  & 2.82 $\pm$ 0.76    & 5.10 $\pm$0.49  \\
\hline
\acrshort{fm}      &   \textbf{0.73 $\pm$ 0.09}~~ & 4.43 $\pm$ 0.15~~  & 0.155 $\pm$ 0.004~~  & 4.80 $\pm$ 1.14~~   & 8.52 $\pm$0.59~~  \\
\acrshort{pk}      &  4.39 $\pm$ 5.50~~ & \textbf{4.14 $\pm$ 0.14}~~  & \textbf{0.100 $\pm$ 0.005}~~  & 41.9 $\pm$ 77.2~~   & 7.95 $\pm$2.42~~  \\
\acrshort{ffq}     &  5.49 $\pm$ 16.5~~ &5.83 $\pm$ 1.39 ~~  & \textbf{0.102 $\pm$ 0.007}~~  & \textbf{4.59 $\pm$ 2.74}~~   & 5.58 $\pm$0.48~~  \\
\hline
\lopon    &  \textbf{0.82$\pm$ 0.18}~~ & \textbf{4.24 $\pm$ 0.18}~~  & \textbf{0.099 $\pm$ 0.006}   & \textbf{4.05 $\pm$ 2.13}~~   & \textbf{5.20 $\pm$0.62}~~  \\
\hline
\end{tabular}

}
\end{center}
\end{table*}

\begin{table*}[t]
\caption{Error rates of different models on classification tasks}
\label{tab:classfication}
\begin{center}
\scalebox{1.0}{
\begin{tabular}{c|ccccccc}	
methods & mnist        & fashion-mnist  & skin             &sensIT          & letter           &covtype-b       & covtype \\	
\hline	
\hline	
\acrshort{ffr}   & 0.0185          & 0.108 & 0.0313           & 0.176          & 0.096            & 0.113          &0.146 \\	
\hline	
\acrshort{fm}   & 0.0573          & 0.167          & 0.0439           & 0.260          & 0.546            & 0.208          &0.575 \\	
\acrshort{pk}   & 0.0506          & 0.168             & 0.0039           & 0.225          & 0.248            & 0.191          &0.494 \\	
\acrshort{ffq}  & 0.0503          & 0.127          & \textbf{0.0018}           & 0.199          & 0.104            & \textbf{0.097} & \textbf{0.103} \\	
\hline	
\lopon & \underline{\textbf{0.0171}} & \textbf{0.117}          &  \underline{\textbf{0.0017}}  & \textbf{0.175} &  \underline{\textbf{0.0729}}  & 0.117          &0.140 \\	
\hline	
\end{tabular}

}
\end{center}
\end{table*}

\begin{table*}[t]
\caption{Effect of batch normalization and dropout}
\label{tab:concrete-bn-dropout}
\begin{center}
\begin{tabular}{c|ccccc}
\hline
$L$ & 1 & 2 & 3 & 5 & 10 \\
\hline
neither          
& \textbf{7.76 $\pm$ 0.53}
& \textbf{6.30 $\pm$ 0.96}
& 6.89 $\pm$ 2.06         
& 7.86 $\pm$ 3.10         
& 7.49 $\pm$ 2.77             \\
only dropout     
& \textbf{7.78 $\pm$ 0.54}
& \textbf{5.98 $\pm$ 0.57}
& \textbf{5.12 $\pm$ 0.58}
& \textbf{4.92 $\pm$ 0.78}
& \textbf{4.71 $\pm$ 0.90}    \\
only BN          
& \textbf{7.82 $\pm$ 0.55} 
& \textbf{6.26 $\pm$ 0.87} 
& 5.82 $\pm$ 1.13          
& \textbf{5.46 $\pm$ 1.83} 
& \textbf{4.97 $\pm$ 0.97}   \\
BN and dropout               
& \textbf{7.77 $\pm$ 0.53}
& \textbf{6.05 $\pm$ 0.50}
& \textbf{5.20 $\pm$ 0.62}
& \textbf{4.72 $\pm$ 0.66}
& \textbf{4.58 $\pm$ 0.74}   \\
\hline
\end{tabular}

\end{center}
\end{table*}

\begin{table*}[ht]
  \begin{minipage}[t]{0.6\linewidth}
	\caption{Effect of dropout and BN on the mnist dataset}
    \label{tab:mnist-bn-dropout}
\scalebox{0.88}{
	\begin{tabular}{c|ccccc}
\hline
$L$ & 1 & 2 & 3 & 5 & 10 \\
\hline
neither
& \textbf{0.0171}
& \textbf{0.0192}
& \textbf{0.0207}
& 0.8947         
& \textbf{0.0241} \\
only dropout
& 0.0208             
& \textbf{0.0188}    
& \textbf{0.0187}    
& 0.7657             
& 0.0298          \\
only BN 
& 0.0242         
& \textbf{0.0202}
& 0.0229         
& 0.0271         
& \textbf{0.0207} \\
BN and dropout    
& \textbf{0.0191}
& \textbf{0.0191}
& \textbf{0.0170}
& \textbf{0.0207} 
& \textbf{0.0230} \\
\hline
\end{tabular}

}
  \end{minipage}%
  \hfill
  \begin{minipage}[t]{0.4\linewidth}
    \captionof{figure}{The distribution of the network outputs. } 
    \label{fig:approx}
    \centering
\includegraphics[height=0.35\textwidth]{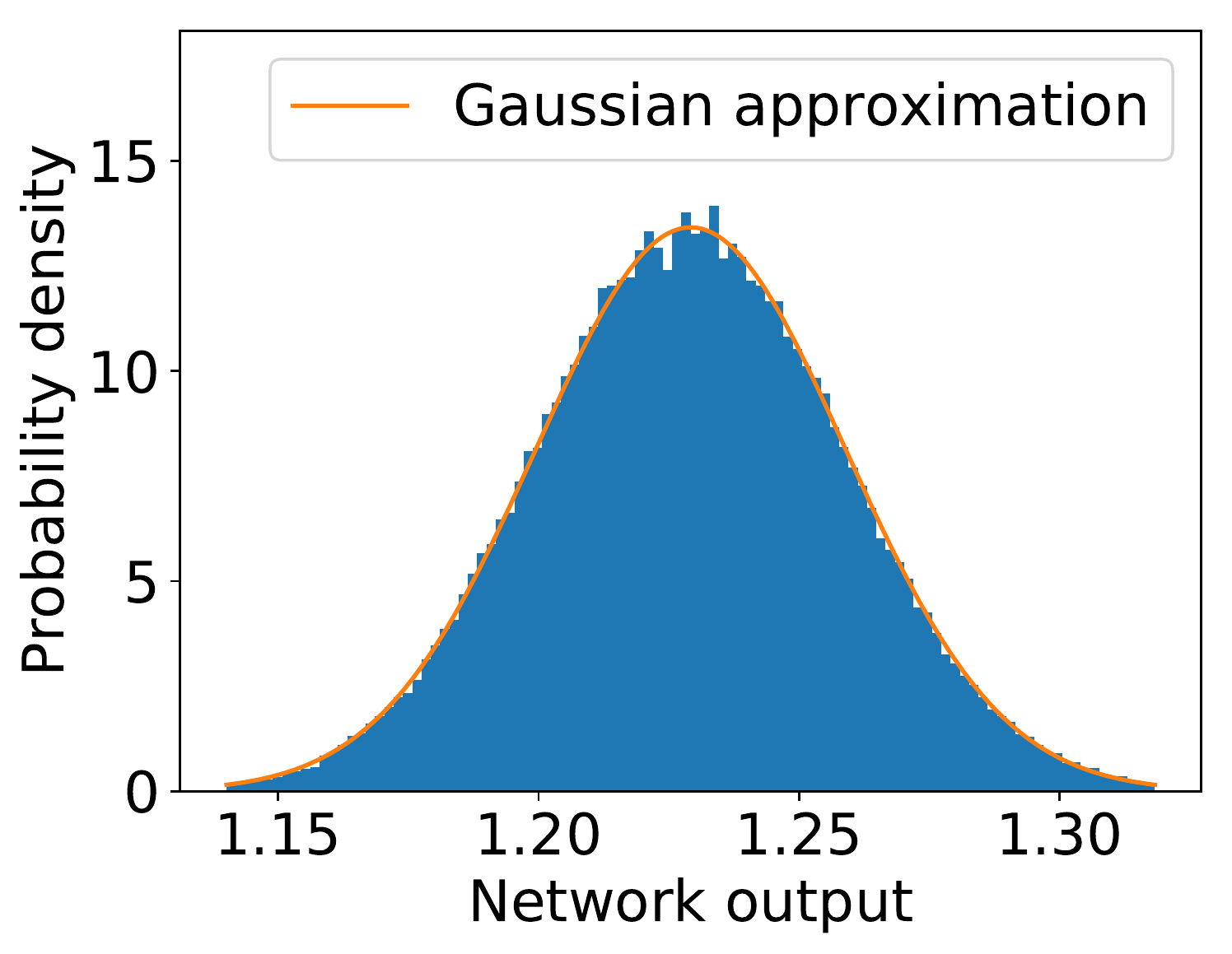}
\hspace{0.3cm}
\includegraphics[height=0.35\textwidth]{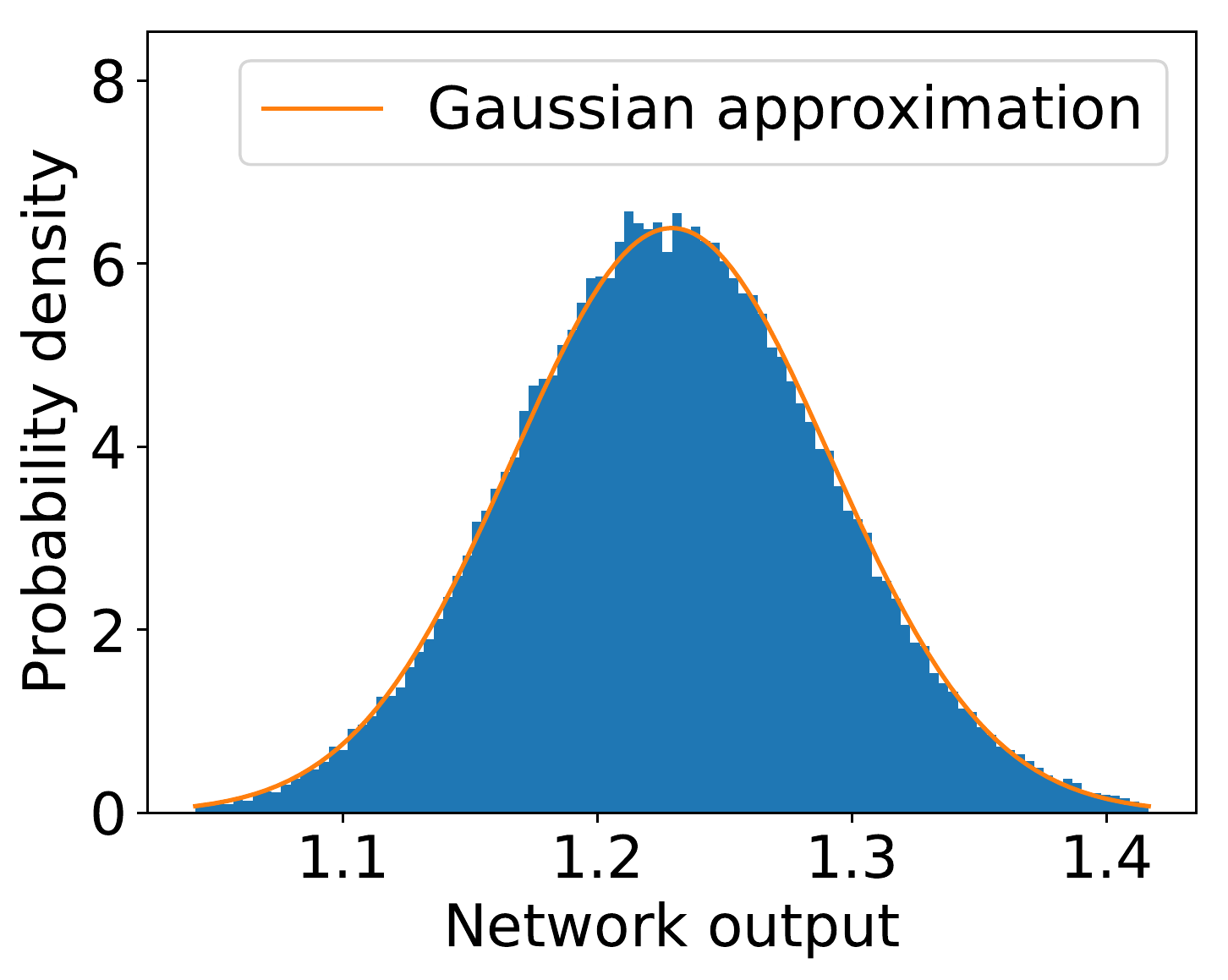}
  \end{minipage}
\end{table*}

\subsection{The product activation}

We first check the uniqueness of the product activation. In this experiment, we check how well a two-layer feedforward neural network can fit the multiplication function. If they cannot easily fit such a function, then it means common activations are unlikely to imitate a product activation.  We randomly generate two features $(x_1, x_2)$ from the range $[-1, 1] \times [-1, 1]$ and then define the multiplication $y_p = 4 x_1 x_2$ to represent a product activation. With the constant, the average of the absolute value of $y_p$ is 1. We use a feedforward network with the $\mathrm{\tanh}$ activation, $y_{f} = \bw^{\top} \mathrm{\tanh}(\bW [x_1, x_2]^\top + \bb_1) + b_2 $,  to fit the multiplication $y_p$. Here the number of rows in $\bW$ is the number of hidden units in this one-layer neural network. We train the neural network  for 200 epochs and record the RMSE. We run the experiment for 10 times and take the average and variance of 10 RMSE values.    

The results are tabulated in the first row of Table \ref{tab:approx_prod}. It is actually hard for the feedforward neural network to fit the product operation. The feedforward neural network needs to use 4 hidden units to get satisfying results. As a reference, it is much easier for the same network to fit a ReLU activation . The ReLU activation is incorporated into a function $y_r = \mathrm{relu}(0.5x_1 + 1.5x_2) / C$, with $C$ normalizing $y_r$ to have an average absolute value of 1. The feedforward network has much smaller error when fitting the function $y_r$ (see RMSEs in the second row of Table \ref{tab:approx_prod}). When the product activation is an effective way to combine two features through multiplication, it is not easy for other activations like $\mathrm{tanh}$ to appproximate the same operation. This result indicates that some information processing by product activations are actually hard for common activations.    

We then examine the product activation function $\bh = \sigma_p(\bu; \bV, \bx)$ in a trained model. We set up a \lopon with three hidden layers and then train it on the mnist dataset. The training finishes after 20 epoches when the network has a validation accuracy of 0.984. Then we check the input $\bu$ and the response $\bh$ of the activation functions at three different layers. We plot the response $h_i$ against the corresponding $u_i$ for each hidden unit $i$ to generate a subplot. We randomly select 400 instances and plot each $(h_i, u_i)$ pair.  We plot four hidden units at each of all three hidden layers and generate plots in Figure \ref{fig:product_act}.

In these results, we see that the product activation is not really a function because inputs with the same value  
may have different responses. It is also clear that the product activation is not linear. The behavior of the activation is versatile: 
the activations shown in (c) at layer 1 and (b) at layer 3 exhibit some linear behavior while the activations shown in (d) at layer 1 
and (c) at layer 2 roughly approximate the quadratic function.

\subsection{Regression and Classification}

In this section, we evaluate the \lopon on several regression and classification tasks. The \lopon is compared against feedforwrad network and three polynomial learning models. All models are summarized below.  

\textit{Feedforward network (\acrshort{ffr}):} feedforward networks uses ReLU functions as activations. We add $l$-2 norm regulraization to the model. The regularization weight is chosen from \{1e-6, 1e-5, 1e-4, 5e-4\}. When dropout is applied, the dropout rate is chosen from \{0, 0.05, 0.1, 0.2, 0.4\}.

\textit{Polynomial Neural Network (\acrshort{ffq}):} the model is the same as the \acrshort{ffr} except its activations are the quadratic function. It has the same hyperparameters as \acrshort{ffr}, and it is trained in the same way as \acrshort{ffr}.

\textit{Factorization Machine (\acrshort{fm}):} we use the implementation from the \texttt{sklearn} package \citep{niculae2019}. The order of \acrshort{fm} in this implementation can be 2 or 3. It add several ANOVA kernel functions (called factors) to increase model complexity. The model is also regularized by $l$-2 norm. The hyperparameters of \acrshort{fm} include the order, the number of factors, and the weight of regularization. The number of factors is chosen from from \{2, 4, 8, 16\}, and the regularization weight is chose from the same range as \acrshort{ffr}. This implementation of \acrshort{ffr} does not have multiple outputs, so we have used one-vs-rest for multiclass classification problems.   
 
\textit{Polynomial Kernel (\acrshort{pk}):} \acrshort{pk} uses polynomial kernels. Other than that, \acrshort{pk} is similar to \acrshort{fm}. We can specify the order of the underlying polynomial function of \acrshort{pk}. The hyperparameters of \acrshort{pk} are the same as \acrshort{fm}. The implementation is also from the \texttt{sklearn} package.   

\textit{\lopon:} the model is the same as the \acrshort{ffr} except its activations are product activations. Its hyperparameters are the same as \acrshort{ffr}, and it is trained in the same way as \acrshort{ffr}.

We test these models on five regression datasets (wine-quality, power-plant, kin8nm, boston-housing, and concrete-strength) and six classification 
datasets (mnist, fashion-mnist, skin, sensIT, letter, covtype-b, and covtype). The mnist and fasion-mnist datasets come with the 
Keras package, the skin, sensIT, and covtype-b datasets are from the libSVM website, and all other datasets are from the UCI repository.  

\parhead{Regression.} We first apply the model to five regression tasks. We use the same data splits by \citet{gal2019repo}. Each dataset has 20 random splits. On each split, we run model selection through five-fold cross validation, re-train the model, and then test the model on the test set. The results are averaged over the 20 splits. For \acrshort{ffr}, \acrshort{ffq}, and \lopon, we set three hidden layers and 50 hidden units in each hidden layer. We apply dropout and batch normalization to all the three models. We set the polynomial order to be 4 for the \acrshort{pk} model to match the order as \lopon. We set the order of \acrshort{fm} to be 3. For each model, we select all hyperparameters described in the subsection above.

Table \ref{tab:regression} tabulates RMSE of all algorithms on all datasets. Each entry is the average RMSE over 20 splits and its standard deviation. We compare \lopon against the competing polynomial models with paired $t$-tests, and the performance(s) of the best polynomial model(s) is bolded here. In general, \lopon performs better than other polynomial models. The \acrshort{ffq} has very bad performances on two splits of the wine-quality dataset. We speculate that \acrshort{ffq} is not stable when its polynomial order is high. \acrshort{pk} has bad performances on wine-quality and boston-housing because the model does not fit the two tasks-- its performances are bad on most splits. The performance of \lopon is slightly worse than the perforamnce of \acrshort{ffr}.

 \parhead{Classification.} We then test these models on seven classification tasks. For each dataset, we set 30\% as the test set, except for mnist and fasion-mnist datasets, which come with test sets. We do model selection for both architecture and hyperparameters on 20\% of the training set. For neural networks, the number of hidden layers is chosen from $\{1, 2, 4\}$. We shrink the number of hidden units from the bottom to the top. The number of hidden units is computed by $\alpha^{\ell} (d_{out} - d_{in}) + d_{out}$ so that the number of hidden units in a middle layer is between the input dimension and the output dimension. The shrinking factor $\alpha$ is chose from $\{0.3, 0.5, 0.7, 0.8\}$. We also select the order for \acrshort{pk} from $\{2, 3, 5\}$ to match the order of \lopon. All other hyperparameters of a model are also selcted together with architectures.
 
The error rates of different models are reported in Table \ref{tab:classfication}. We omit standard deviations in the table to save space. Polynomial models are compared with paired $t$-tests. The performance of the best polynomial model is bolded. We also compare \lopon against \acrshort{ffr}: if \lopon is significantly better than \acrshort{ffr}, we underline the \lopon's performance. The performance of \lopon is better other polynomial models in general. \lopon is comparable to \acrshort{ffr} on classification tasks. We speculate the reason is that an \lopon only needs to decide discrete labels from its outputs in classification tasks while it needs to fit the exact value in regression tasks. \lopon may be not flexible enough for fitting continuous values compared to feedforward networks.

In a summary, the \lopon narrows the performance gap between polynomial models and the well-studied feedforward neural networks, making the polynomial more practical in typical learning tasks.

\subsection{The effect of batch normalization and dropout}

In this subsection, we investigate the effects of batch normalization and dropout on \lopon. We use $L \in \{1, 2, 3, 5, 10 \}$. For each depth, we try four combinations: using/not using batch normalization and/or dropout. We select other hyperparameters through model selection. We run the experiment on a regression task (concrete-strength) and a classification task (mnist).  

The results are shown in Table \ref{tab:concrete-bn-dropout} and \ref{tab:mnist-bn-dropout}. For each depth $L$, the four combinations are compared with paired $t$-tests, and the best performance(s) across four combinations are bolded. From this result, we see that both batch normalization and dropout are needed to train a good \lopon model when the model is deep. On the mnist dataset, the \lopon without batch normalization has very bad performance when $L=5$. Its performance drops sharply after a few epochs. This observation indicates that the \lopon without batch normalization is very unstable due to some bad optimization directions.

We also have the following observations in training: batch normalization in training tends to increase the depth of the model; and dropout tends to decrease the scale of weight matrices. These observations are consistent with previous observations in the literature.  

\subsection{Network outputs with stochastic network weights}

In this experiment we check the network output $\bh^{L}$ when the network parameters in $\theta$ are from a Gaussian prior. We use a \lopon with  $L=5$ layers. The mean of the prior is given by network weights of a trained \lopon. The variance of the prior is set to $\sigma^2 \bI$. Then we sample 10,000 samples from the Gaussian prior as network weights to compute samples of $\bh^{L}$. Here $\bh^{L}$ has only one entry, so we can plot its samples into a histogram. At the same time, we compute the moments of the distribution of $\bh^{L}$ and get a Gaussian approximation. Figure \ref{fig:approx} shows the histograms and the corresponding approximate Gaussian distributions for $\sigma^2=0.05$ (left) and  $\sigma^2=0.1$ (right). From this result, we see that the Gaussian approximation is very accurate.

\section{Conclusion}

In this paper, we propose the product activation and use it to construct the \lopon,  a new type of polynomial neural networks. As a neural network, a \lopon can be trained with modern training techniques, such as dropout and batch normalization. As a decomposition model, it connects other decomposition models with neural networks. Now we have a new approach to convert other decomposition models to neural networks that have a similar structure as a \lopon.  

The network function of \lopon is multilinear in its parameters. With this property,  we can efficiently calculate moments of the network's outputs when the network is given a prior. The moments allow us to approximate the distribution of the network outputs, then we will be able to approximately maximize the likelihood of the data without using complex inference algorithms.  

With its unique theoretical properties and competitive performances in practice, the \lopon a valuable learning model.

\section*{Acknowledgement}

The work was supported by NSF 1850358. 

\bibliography{ref-liping}
\bibliographystyle{icml2020}

\end{document}


\onecolumn
\icmltitle{The Product Activation and Low-Order Polynomial Neural Networks}



\icmlsetsymbol{equal}{*}

\begin{icmlauthorlist}
\icmlauthor{Aeiau Zzzz}{equal,to}
\icmlauthor{Bauiu C.~Yyyy}{equal,to,goo}
\icmlauthor{Cieua Vvvvv}{goo}
\icmlauthor{Iaesut Saoeu}{ed}
\icmlauthor{Fiuea Rrrr}{to}
\icmlauthor{Tateu H.~Yasehe}{ed,to,goo}
\icmlauthor{Aaoeu Iasoh}{goo}
\icmlauthor{Buiui Eueu}{ed}
\icmlauthor{Aeuia Zzzz}{ed}
\icmlauthor{Bieea C.~Yyyy}{to,goo}
\icmlauthor{Teoau Xxxx}{ed}
\icmlauthor{Eee Pppp}{ed}
\end{icmlauthorlist}

\icmlaffiliation{to}{Department of Computation, University of Torontoland, Torontoland, Canada}
\icmlaffiliation{goo}{Googol ShallowMind, New London, Michigan, USA}
\icmlaffiliation{ed}{School of Computation, University of Edenborrow, Edenborrow, United Kingdom}

\icmlcorrespondingauthor{Cieua Vvvvv}{c.vvvvv@googol.com}
\icmlcorrespondingauthor{Eee Pppp}{ep@eden.co.uk}

\icmlkeywords{Machine Learning, ICML}

\vskip 0.3in



\printAffiliationsAndNotice{\icmlEqualContribution} 

\setcounter{equation}{15}

\section*{Appendix}

Smoothness of a neural network function is one way to understand the robustness of the neural network in adversarial learning. We have also analyzed the smoothness of the \lopon function.  

\subsection*{Smoothness of the function surface}

The smoothness of the function surface of a neural network characterizes the network's important properties, such as being robust to input perturbations. This subsection gives an upper bound of the Lipschitz constant of the \lopon function. 

We first compute the first order derivative of $\bh^\ell(x)$ with respective to the input $\bx$. According to (5), we have,
\begin{align*}
\nabla \bh^1 &= \diag(\bV^1 \bx) \bW^1 + \diag(\bW^{1} \bx) \bV^{1}, \\
\nabla \bh^\ell &= \diag(\bV^\ell \bx) \bW^\ell (\nabla \bh^{\ell-1}) + \diag(\bW^{\ell} \bh^{\ell-1}) \bV^{\ell}, \nonumber\\
& \hspace{0.2in} \ell=2, 3, \cdots, L. 
\end{align*}

Then we can give a bound of the first order derivative, which in turn gives an estimate of the Lipschitz constant of $\bh^\ell$ as a function of $\bx$ .  The results are summarized in the following theorem. Here, for the sake of simplicity, we use $l^2$ norm ($\| \cdot \|$) for both vectors and matrices. 

\begin{theorem}
If $\bh^\ell$ is defined recursively as in (1) and (2), then, for $\ell=1,2,\ldots, L$, 
\begin{align}\label{ine:h-bound}
\| \bh^\ell \| &\leq \left( \prod_{k=1}^\ell \| \bV^k\| \| \bW^{k} \| \right) \| \bx \|^{\ell+1} 
\end{align}
and
\begin{align}\label{ine:grad-h-bound}
\| \nabla \bh^\ell(\bx) \| &\leq (\ell+1) \left(\prod_{k=1}^{\ell} \| \bV^k \| \| \bW^k \| \right) \|\bx\|^{\ell} 
\end{align}
where the matrix norm $\| \cdot \|$ is the maximal singular value of its input matrix.
\end{theorem}

\begin{proof}
Our proof is based on two inequalities. The first one is from the definition of 
matrix norm, e.g.~$\|\bW^\ell \bh^{\ell - 1} \| \le \rho\left( \bW^\ell \right)  \|\bh^{\ell - 1}\|$. 
The second one is the Cauchy-Schwarz inequality, e.g.~$\|(\bW^\ell \bh^{\ell - 1}) \odot (\bV^\ell \bx)\| \le  
\|(\bW^\ell \bh^{\ell - 1}) \| \cdot \| (\bV^\ell \bx)\| $.  

We will use mathematical induction to derive the results. For $\ell=1$,
\begin{align*}
\| \bh^1 \| &= \| \diag(\bV^1 \bx) \bW^1 \bx\| \leq \|\bV^1 \| \| \bW^1 \| \|\bx\|^2, \\
\| \nabla \bh^1 \| &= \| \diag(\bV^1 \bx) \bW^1 + \diag(\bW^1 \bx) \bV^1 \|  \leq 2 \| \bV^1 \| \| \bW^1 \| \|\bx \|.
\end{align*}
This means~\eqref{ine:h-bound} and~\eqref{ine:grad-h-bound} hold for $\ell=1$.  Now assume~\eqref{ine:h-bound} and~\eqref{ine:grad-h-bound} hold for $\ell-1$, then 
\begin{align*}
\| \bh^\ell \| & = \| \diag(\bV^\ell \bx) \bW^{\ell} h^{\ell-1} \| \leq \| \bV^\ell \| \|\bx \| \|\bW^\ell\| \|\bh^{\ell-1}\| \leq \left(\prod_{\ell=1}^{k} \| \bV^\ell \| \|\bW^\ell\| \right) \|\bx\|^{\ell+1},
\end{align*}
and 
\begin{align*}
\| \nabla \bh^\ell \| & =\| \operatorname{diag}(\bV^\ell \bx) \bW^\ell (\nabla \bh^{\ell-1}(\bx)) + \operatorname{diag}(\bW^\ell \bh^{\ell-1}(\bx)) \bV^\ell \| \\
	& \leq \| \bV^\ell \| \|\bx\| \|\bW^\ell\| \| \nabla \bh^{\ell-1}(\bx) \| + \| \bW^\ell \| \| \bh^{\ell-1}(\bx)\| \|\bV^\ell\|   \\
	& = (\ell+1) \left( \prod_{k=1}^{\ell} \|\bV^k \| \|\bW^{k}\|  \right) \|\bx\|^\ell.
\end{align*}
\end{proof}

This bound indicates the relation between the smoothness and the norm of weight matrices. The result is similar to the feedforward neural networks with other activation functions whose gradients are always between -1 and 1 \citep{virmaux2018lipschitz}. 
As a comparison, an \acrshort{ffq} is much less smooth in general, as its order is exponential in its number of layers.

The proof above is independent of the data and network parameters. We can further improve the bound by considering network weights with an approach similar to  \citep{virmaux2018lipschitz}, but we defer the further investigation to future work.

\bibliography{ref-liping}
\bibliographystyle{icml2020}


%

%

\onecolumn
\title{Supplementary Materials for the Submission ``Polynomial \\ Functions are Strong Learners''}

\section{Missing Proofs}
\subsection{Proof of Theorem 4.1}

\setcounter{section}{4}

\begin{theorem} 
The learning models in the form of $y(\bx) = \sum_{k=1}^K \pi_k (\lambda + \bp_k^\top \bx)^m$ \citep{blondel2016polynomial} can be written as a \lopon function. 
\end{theorem}
\setcounter{section}{1}

\begin{proof}
By Appendix D.3 in \citep{blondel2016polynomial} , the function of a factorization machine can be computed by 
\begin{align*}
FM(\bx; \bU, \bS) 
&= \frac{1}{2}\left[\bx^\top \bU \bS^\top \bx -  (\bx \odot \bx)^\top \diag\left(\bS \bU^\top \right) \right] 
\end{align*}

We can set $\bW^1$ and $\bV^1$ and get $\bh^1$ as follows. 
\begin{align*}
\bW^1 = \left[ \begin{array}{c}
\bU \\
\bI
\end{array} \right],  \bV^1 = \left[ \begin{array}{c}
\bS \\
\bI
\end{array} \right], 
\bh^1 = \left[ \begin{array}{c}
(\bU \bx) \odot (\bS \bx) \\
\bx \odot \bx
\end{array} \right]
\end{align*}
Then we set $\bW^2 = [\bone^\top, - \diag\left(\bS \bU^\top \right)]$, then $\bW^2 \bh^1 = FM(\bx; \bU, \bS)$. \end{proof}

\subsection{Proof of Theorem 4.2}

\setcounter{section}{4}
\begin{theorem}
The second-order factorization machines taking the form
$y(\bx) = w_0 + \bw_1^\top \bx + \sum_{i=1}^{d_0}\sum_{j =i+1}^{d_0} \bv_i^\top \bv_j x_i x_j$ \citep{rendle2010factorization} can be written as a \lopon function.   
\end{theorem}
\setcounter{section}{1}

\begin{proof} 
We first write the polynomial kernel function, $\calP^m(\bp, \bx) = (\lambda + \bp^\top \bx)^m, \bp, \bx \in \calR^{d}$ in the form of $\loponfunc(\bx; \theta)$. Append an element 1 to the feature vector $\bx$ as the new input $[\bx, 1]$ to the network. Set $\bW^1 = [\bp^\top, \lambda]$, $\bV^\ell = [\bp^\top, \lambda]$ for all $\ell = 1, \ldots, m-1$. They are all matrices with only one row. Set $\bW^{\ell}$ to be a single element matrix $[1]$ for $\ell = 2, \ldots, L$, then $\loponfunc(\bx; \theta)$ is equivalent to the kernel by (6). 

To get a weighted sum of multiple kernels, we can just put each kernel in each row of $\bW^1$ and every $\bV^{\ell}$. Then we set all $\bW^{\ell}, \ell =  2, \ldots, L$ to the identity matrix  with size $K$. Then all these kernels work in parallel in the \lopon network. Finally, we add an extra layer to take the weighted sum of the values of these kernels.
\end{proof}

\section{Smoothness of the function surface}

The smoothness of the function surface of a neural network characterizes the network's important properties, such as being robust to input perturbations. This subsection gives an upper bound of the Lipschitz constant of the \lopon function. 

We first compute the first order derivative of $\bh^\ell(x)$ with respective to the input $\bx$. According to (5), we have,
\begin{align*}
\nabla \bh^1 &= \diag(\bV^1 \bx) \bW^1 + \diag(\bW^{1} \bx) \bV^{1}, \\
\nabla \bh^\ell &= \diag(\bV^\ell \bx) \bW^\ell (\nabla \bh^{\ell-1}) + \diag(\bW^{\ell} \bh^{\ell-1}) \bV^{\ell}, \nonumber\\
& \hspace{0.2in} \ell=2, 3, \cdots, L. 
\end{align*}

Then we can give a bound of the first order derivative, which in turn gives an estimate of the Lipschitz constant of $\bh^\ell$ as a function of $\bx$ .  The results are summarized in the following theorem. Here, for the sake of simplicity, we use $l^2$ norm ($\| \cdot \|$) for both vectors and matrices. 

\begin{theorem}
If $\bh^\ell$ is defined recursively as in (1) and (2), then, for $\ell=1,2,\ldots, L$, 
\begin{align}\label{ine:h-bound}
\| \bh^\ell \| &\leq \left( \prod_{k=1}^\ell \| \bV^k\| \| \bW^{k} \| \right) \| \bx \|^{\ell+1} 
\end{align}
and
\begin{align}\label{ine:grad-h-bound}
\| \nabla \bh^\ell(\bx) \| &\leq (\ell+1) \left(\prod_{k=1}^{\ell} \| \bV^k \| \| \bW^k \| \right) \|\bx\|^{\ell} 
\end{align}
where the matrix norm $\| \cdot \|$ is the maximal singular value of its input matrix.
\end{theorem}

\begin{proof}
Our proof is based on two inequalities. The first one is from the definition of 
matrix norm, e.g.~$\|\bW^\ell \bh^{\ell - 1} \| \le \rho\left( \bW^\ell \right)  \|\bh^{\ell - 1}\|$. 
The second one is the Cauchy-Schwarz inequality, e.g.~$\|(\bW^\ell \bh^{\ell - 1}) \odot (\bV^\ell \bx)\| \le  
\|(\bW^\ell \bh^{\ell - 1}) \| \cdot \| (\bV^\ell \bx)\| $.  

We will use mathematical induction to derive the results. For $\ell=1$,
\begin{align*}
\| \bh^1 \| &= \| \diag(\bV^1 \bx) \bW^1 \bx\| \leq \|\bV^1 \| \| \bW^1 \| \|\bx\|^2, \\
\| \nabla \bh^1 \| &= \| \diag(\bV^1 \bx) \bW^1 + \diag(\bW^1 \bx) \bV^1 \|  \leq 2 \| \bV^1 \| \| \bW^1 \| \|\bx \|.
\end{align*}
This means~\eqref{ine:h-bound} and~\eqref{ine:grad-h-bound} hold for $\ell=1$.  Now assume~\eqref{ine:h-bound} and~\eqref{ine:grad-h-bound} hold for $\ell-1$, then 
\begin{align*}
\| \bh^\ell \| & = \| \diag(\bV^\ell \bx) \bW^{\ell} h^{\ell-1} \| \leq \| \bV^\ell \| \|\bx \| \|\bW^\ell\| \|\bh^{\ell-1}\| \leq \left(\prod_{\ell=1}^{k} \| \bV^\ell \| \|\bW^\ell\| \right) \|\bx\|^{\ell+1},
\end{align*}
and 
\begin{align*}
\| \nabla \bh^\ell \| & =\| \operatorname{diag}(\bV^\ell \bx) \bW^\ell (\nabla \bh^{\ell-1}(\bx)) + \operatorname{diag}(\bW^\ell \bh^{\ell-1}(\bx)) \bV^\ell \| \\
	& \leq \| \bV^\ell \| \|\bx\| \|\bW^\ell\| \| \nabla \bh^{\ell-1}(\bx) \| + \| \bW^\ell \| \| \bh^{\ell-1}(\bx)\| \|\bV^\ell\|   \\
	& = (\ell+1) \left( \prod_{k=1}^{\ell} \|\bV^k \| \|\bW^{k}\|  \right) \|\bx\|^\ell.
\end{align*}
\end{proof}

This bound indicates the relation between the smoothness and the norm of weight matrices. The result is similar to the feedforward neural networks with other activation functions whose gradients are always between -1 and 1 \citep{virmaux2018lipschitz}. 
As a comparison, an \acrshort{ffq} is much less smooth in general, as its order is exponential in its number of layers.

The proof above is independent of the data and network parameters. We can further improve the bound by considering network weights with an approach similar to  \citep{virmaux2018lipschitz}, but we defer the further investigation to future work. 

\bibliography{ref-liping}
\bibliographystyle{icml2020}